\documentclass[letterpaper]{article}
\usepackage[preprint]{aaai2027}
\usepackage[hyphens]{url}
\usepackage{graphicx}
\usepackage{natbib}
\usepackage{caption}
\usepackage{booktabs}
\usepackage{multirow}
\usepackage{amsmath,amssymb,amsthm,bm}
\usepackage{xcolor}
\usepackage{tikz}
\newcommand{\val}[2]{#1\,{\scriptsize$\pm$#2}}
\newcommand{\E}{\mathbb{E}}
\newcommand{\Var}{\operatorname{Var}}
\newtheorem{assumption}{Assumption}
\newtheorem{lemma}{Lemma}
\newtheorem{proposition}{Proposition}
\theoremstyle{remark}
\newtheorem*{remark}{Remark}

\title{SpecBraM: What Should an EEG Foundation Model Predict?\\
Masked Band-Power Prediction versus Waveform Reconstruction}
\author{Peng Xie\textsuperscript{\rm 1}, Yequan Bie\textsuperscript{\rm 1}, Jianda Mao\textsuperscript{\rm 1}, Kani Chen\textsuperscript{\rm 1}}
\affiliations{\textsuperscript{\rm 1}The Hong Kong University of Science and Technology}

\begin{document}
\maketitle

\begin{figure*}[!t]\centering
\begin{tikzpicture}
\node[anchor=south west,inner sep=0] (img) {\includegraphics[width=0.98\textwidth]{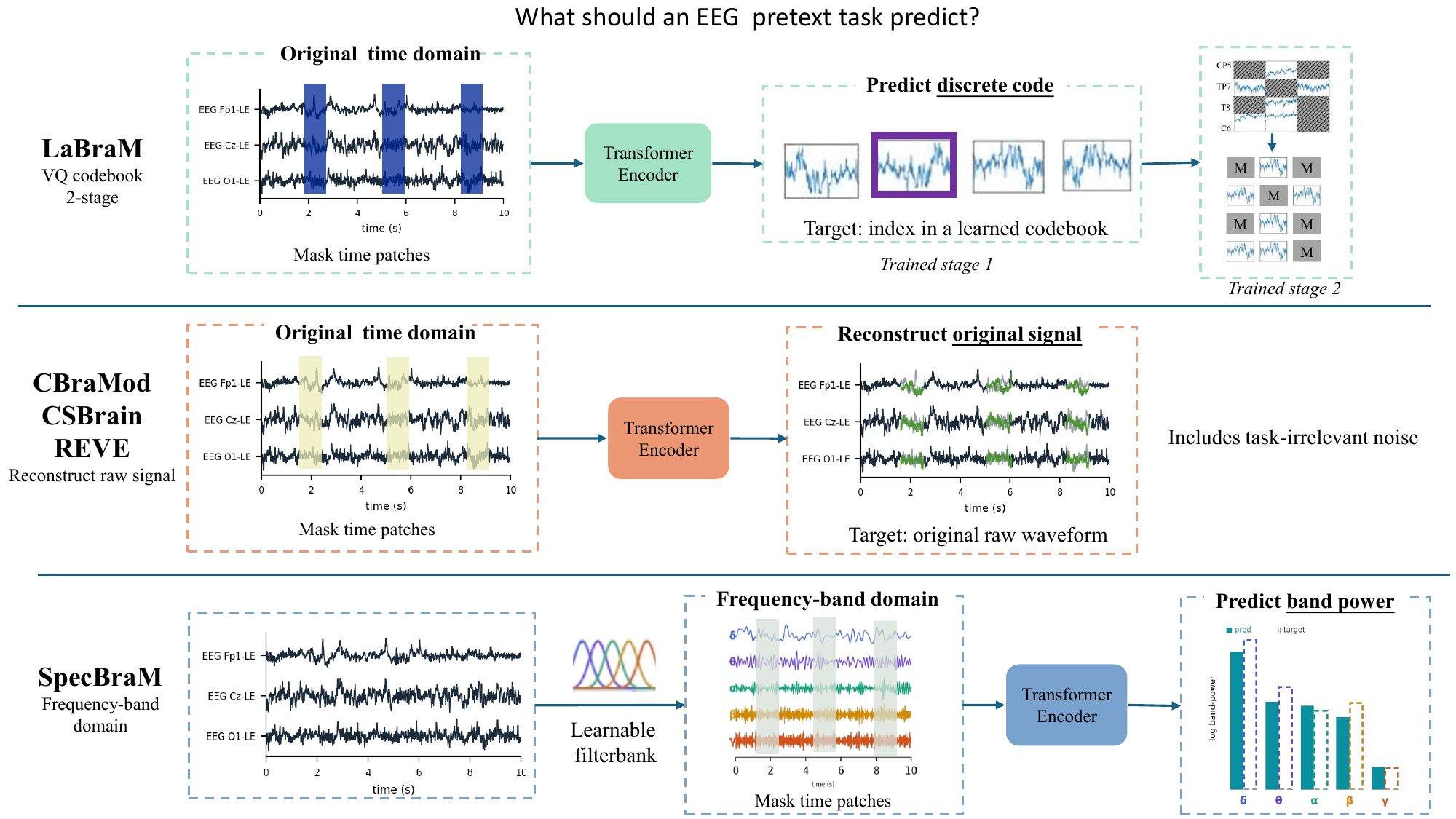}};
\begin{scope}[x={(img.south east)},y={(img.north west)}]
\node[anchor=north west,fill=white,draw=black!60,rounded corners=1pt,inner sep=2pt,
      font=\scriptsize\sffamily,align=left] at (0.005,0.30)
      {shown: filterbank-tokenizer variant\\main model: linear patch embedding};
\end{scope}
\end{tikzpicture}
\caption{\textbf{What should an EEG pretext task predict?} LaBraM predicts discrete codes
from a separately trained vector-quantized tokenizer; CBraMod, CSBrain and REVE reconstruct
the raw waveform of masked patches. SpecBraM keeps the masked-patch setting but changes the
\emph{target}: it regresses the log spectral energy of each masked patch in a few fixed
frequency bands. The bottom row shows the \emph{filterbank-tokenizer variant}; the main
model uses a plain linear patch embedding. In our experiments the effect of the target is
larger than that of the tokenizer (Sec.~\ref{sec:results}).}
\label{fig:contrast}
\end{figure*}

\begin{abstract}
Self-supervised EEG models often reconstruct masked waveforms or predict discrete codes.
We study a task-aligned alternative: masked band-power prediction (MBP), which predicts
fixed narrow-band log spectral energy for masked channel-time patches. This target retains
rhythm power relevant to sleep staging while avoiding phase-sensitive waveform reconstruction
and a learned codebook. Across three pretraining seeds, we compare band-power and waveform
targets with matched backbones, pretraining data (2,388 hours), and training steps, including
a 2x2 tokenizer-by-target design. On ISRUC and HMC sleep staging, MBP exceeds raw- and
band-waveform reconstruction by 1.6--2.8 balanced-accuracy points with all labels and
4.7--7.3 points with 1\% of labels under a strict linear probe; the target effect exceeds
the tokenizer effect. Its frozen features reach 0.7916/0.7425 balanced accuracy, versus
0.7636/0.7227 for a matched rich handcrafted spectral baseline, although the gap is about
one point with 1\% of labels. Full fine-tuning reaches 0.8107/0.7669. The gains do not
extend to every task with spectral cues, including motor imagery, depression screening, and
vigilance regression. These results support choosing pretraining targets to match the physical
quantities and spatial and temporal scales relevant to downstream labels.
\end{abstract}

\section{Introduction}
EEG is central to clinical neuroscience---sleep staging, screening for abnormal background
activity, studies of depression and cognitive load---but expert labels are scarce, which
motivates foundation models pretrained once on large unlabeled corpora such as the TUH EEG
Corpus (TUEG) \cite{obeid2016tuh} and transferred to many tasks
\cite{jiang2024labram,wang2025cbramod,csbrain,reve2025}. The dominant designs borrow their
pretext task from other modalities. LaBraM \cite{jiang2024labram} trains a vector-quantized
tokenizer and predicts its discrete codes, as BEiT does for images \cite{bao2022beit}.
CBraMod \cite{wang2025cbramod}, CSBrain \cite{csbrain} and REVE \cite{reve2025} reconstruct
the raw waveform of masked patches, as masked autoencoders reconstruct pixels \cite{he2022mae}.

We ask what an EEG pretext target \emph{should} be. Our answer is a principle rather than a
new architecture: \emph{predict the physical quantity that the downstream readouts depend on,
not a proxy of the signal.} In EEG, a large class of readouts is defined by spectral energy:
AASM sleep stages are specified by the amount of delta activity, spindles and alpha
\cite{berry2017aasm}; clinical reports describe ``background slowing''; quantitative EEG uses
band powers and their ratios. Meanwhile the phase of each rhythm within a sub-second patch
carries little information about the state and cannot be inferred from neighbouring patches,
so a waveform-reconstruction loss spends much of its gradient on an irreducible component.
Sec.~\ref{sec:theory} makes this precise under an idealized stationary-Gaussian model.

We test the principle with its simplest instance, \textbf{masked band-power prediction}: on a
standard masked-patch transformer, replace the waveform target by the log energy of each
masked patch in five fixed bands. Our contributions are:
\begin{itemize}
\item \textbf{A controlled target study.} Holding backbone, data, step budget and launcher
fixed, we compare six targets on the filterbank tokenizer and a $2{\times}2$ tokenizer$\times$
target design, each over three pretraining seeds (Tables~\ref{tab:targets}
and~\ref{tab:twobytwo}). Every band-power variant beats both waveform-reconstruction variants
on ISRUC and HMC under a strict linear probe, on every seed. A band-\emph{waveform}
reconstruction control also loses; this is consistent with discarding phase (and the
patch's innovation) being the useful ingredient, though that control also changes the output
dimension and loss scale and so does not isolate the cause on its own.
\item \textbf{Few labels and short training.} The gap widens to 5--7 points with $1\%$ of the
labels, and band power at pretraining epoch 4 already exceeds reconstruction at epoch 16
(Tables~\ref{tab:twobytwo} and~\ref{tab:compute}).
\item \textbf{Comparison with direct spectral features.} Probed identically, SpecBraM exceeds a
rich handcrafted spectral feature set by 2.0--2.8 points (about one point with 1\% of labels),
whereas the waveform-reconstruction model of the same architecture does not exceed it on HMC
or at 1\% of labels (Tables~\ref{tab:lpmain} and~\ref{tab:labeleff}).
\item \textbf{Comparison with existing models.} Under the same strict probe, SpecBraM's frozen
features beat the official CBraMod and LaBraM weights on all seven tasks
(Tables~\ref{tab:lpmain}, \ref{tab:taskbin}--\ref{tab:seedvig}); with fine-tuning it reaches
the best published sleep-staging results on ISRUC and HMC (Table~\ref{tab:ft30}). Because
these models differ in pretraining data and duration, this comparison does not by itself
establish data efficiency; the controlled study above is our main evidence.
\item \textbf{Where the benefit stops.} The advantage does not extend to every task with
spectral cues: it is absent for motor imagery (whose labels depend on spatially localized
$\mu/\beta$ changes), depression screening and vigilance regression. On vigilance,
high-passing the input at 4\,Hz restores band power's lead, consistent with low-frequency
(likely ocular) contamination of the target, although that manipulation changes input and
target together and the mechanism is not yet isolated (Sec.~\ref{sec:boundary}).
\end{itemize}

\noindent We do not claim a new paradigm: masked prediction of hand-crafted features is
established in vision \cite{wei2022maskfeat} and speech \cite{hsu2021hubert}, and spectral
information enters the targets of several EEG models (Sec.~\ref{sec:related}).
This work was completed and submitted for peer review in July 2026; FAME
\cite{yu2026fame} and MANAS-2 \cite{kulkarni2026manas2}, which also use band log-power
targets, appeared afterwards and are concurrent.

\section{Related Work}\label{sec:related}
\noindent\textbf{Pretext targets of EEG foundation models.} Most models pretrain on TUEG and
differ mainly in the masked-prediction target. \emph{Raw-waveform reconstruction}: CBraMod
\cite{wang2025cbramod} with a criss-cross transformer, CSBrain \cite{csbrain} with
cross-scale modeling, REVE \cite{reve2025} with an additional global-token reconstruction;
earlier, BENDR \cite{kostas2021bendr} and Neuro-GPT \cite{cui2024neurogpt} used contrastive
or autoregressive objectives, EEGPT \cite{wang2024eegpt} combines reconstruction with latent
alignment, and NeuroLM \cite{jiang2025neurolm} bridges EEG to language tokens.
\emph{Discrete codes}: LaBraM \cite{jiang2024labram} predicts codes of a VQ tokenizer trained
to reconstruct the Fourier amplitude \emph{and phase} of each patch; BioSerenity-E1
\cite{bettinardi2025bioserenity} quantizes log-multitaper spectra; BandVQ
\cite{sukhbaatar2026bandvq} trains one VQ tokenizer per frequency band and predicts quantized
log-power codes. In intracranial recordings BrainBERT \cite{wang2023brainbert} reconstructs
masked spectrograms and Brant \cite{zhang2023brant} reconstructs raw patches with spectral
input features. BIOT \cite{yang2023biot} tokenizes biosignals with STFT magnitudes.

\noindent\textbf{Concurrent work.} FAME \cite{yu2026fame} observes that the $1/f$ spectrum of
EEG and the $\ell_2$ loss bias masked autoencoders towards low frequencies, and corrects this
by decoding the band-averaged log time--frequency power of the \emph{whole} sequence (masked
and unmasked positions) with per-band standardization and equal band weights. MANAS-2
\cite{kulkarni2026manas2} jointly reconstructs raw waveform patches and six-band log-power
targets and adds a constraint on the RMS energy of the reconstructed waveform. Our target is
close to FAME's, but the motivation and evidence differ: we predict only \emph{masked} patches
from context, frame the choice as discarding unpredictable phase (Sec.~\ref{sec:theory}), and
isolate the effect of the target with matched controls including band-waveform
reconstruction. FAME's per-band standardization is a natural refinement of our target
(Sec.~\ref{sec:limits}).

\noindent\textbf{Hand-crafted targets in other modalities.} MaskFeat \cite{wei2022maskfeat}
showed that predicting HOG features of masked regions outperforms predicting pixels; the first
iteration of HuBERT \cite{hsu2021hubert} predicts clusters of MFCC features. Band power plays
the same role for EEG: a fixed feature that keeps semantics and discards low-level detail.
Latent-predictive methods (I-JEPA \cite{assran2023ijepa}, data2vec \cite{baevski2022data2vec},
EEG2Rep \cite{foumani2024eeg2rep}) predict learned representations instead and need explicit
anti-collapse mechanisms such as SIGReg \cite{balestriero2025lejepa}, which we use as a light
regularizer.

\section{Method}\label{sec:method}
\noindent\textbf{Input and tokens.} A 10-s window at 256\,Hz over the 19 channels of the
10--20 montage, normalized per recording (median and interquartile range), is cut per channel
into non-overlapping patches of $P{=}200$ samples ($0.78$\,s), giving a token grid with
$C{=}19$ channels and $T_p{=}\lfloor 2560/200\rfloor{=}12$ time patches; the last 160 samples
($0.625$\,s) of each window are not used, during pretraining and in every downstream window
(including each of the three 10-s windows of a 30-s sleep epoch). The main model embeds each patch with a single
linear map (as CBraMod does). A variant decomposes each patch with a learnable 5-band
filterbank before embedding (``filterbank tokenizer''); we report both because the target
study crosses tokenizer and target.

\noindent\textbf{Encoder and masking.} A 12-layer criss-cross transformer ($D{=}512$, 8 heads;
51.3M parameters) alternates attention along the channel and time axes, with learned channel
and time position embeddings. Each (channel, time) patch is masked independently with
probability $0.3$; the embeddings of masked patches are replaced by a learned mask token.

\noindent\textbf{Target.} For each patch $\bm{x}_{c,t}\in\mathbb{R}^{P}$ we compute five
energies with \emph{fixed} (non-learnable) band-pass filters $h_b$ applied to the
uncorrupted signal:
\begin{align}
P_{c,t,b}&=\frac1P\sum_{\tau}\bigl((h_b*\bm{x}_c)_{t,\tau}\bigr)^2,\nonumber\\
y_{c,t,b}&=\log\bigl(1+P_{c,t,b}\bigr),\qquad b=1,\dots,5. \label{eq:target}
\end{align}
Each $h_b$ is a 65-tap Hann-windowed cosine (zero-mean, unit norm). Their nominal centres are
2/6/10/21/38\,Hz, but with a 0.25-s kernel the measured responses are those of five narrow
bands peaking at about 4.6/6/10/21/38\,Hz (Table~\ref{tab:filters}); the lowest two overlap
and the slowest activity ($<2$\,Hz) is barely represented. We therefore describe the main
target as \emph{five fixed narrow-band spectral energies}. A variant computes exact
standard-band energies (0.5--4/4--8/8--13/13--30/30--45\,Hz) with a per-patch FFT; it stays
well above both waveform-reconstruction targets but is slightly below the main target on HMC
($-0.6$ points; Table~\ref{tab:targets}), so the advantage over reconstruction does not depend
on this particular filterbank, although the two band-power targets are not interchangeable.
We therefore state the principle at the level of \emph{coarse spectral energy}: the classical
$\delta/\theta/\alpha/\beta/\gamma$ partition, which reflects distinct generators but whose exact
edges are conventions (e.g.\ alpha is given as 8--12 or 8--13\,Hz, sleep scoring adds a
spindle band, and individual alpha frequency varies by several hertz), is one natural
instance, and the narrow bands of the main target are another. Which partition is best, and
whether the difference comes from the band definition or the estimator, is left to future
work.

\noindent\textbf{Objective.} A linear head maps each encoder output $H_{c,t}$ to five
predictions $\hat y_{c,t,b}$. With $\mathcal{M}$ the set of masked patches,
\begin{equation}
\mathcal{L}=\frac{1}{5|\mathcal{M}|}\sum_{(c,t)\in\mathcal{M}}\sum_{b=1}^{5}
\bigl(\hat y_{c,t,b}-y_{c,t,b}\bigr)^2\;+\;0.05\,\mathrm{SIGReg}(H),
\label{eq:loss}
\end{equation}
where SIGReg \cite{balestriero2025lejepa} is an isotropy penalty on the encoder outputs. There
is no reconstruction term and no latent-prediction term. The target must be fixed: if the
target filters are learnable, the loss is minimized by shrinking them to zero
(Appendix~\ref{app:earlier}). The waveform-reconstruction controls replace the first term of
Eq.~\eqref{eq:loss} by the MSE between a linear decoder output and the raw samples of each
masked patch (``raw'') or the five fixed band-filtered waveforms of the patch
(``band-waveform''), with everything else unchanged.

\section{Why Band Power: an Idealized Analysis}\label{sec:theory}
We formalize the motivation under an idealized model; the full statements and proofs are in
Appendix~\ref{app:theory}. The analysis is a motivation and a source of testable predictions,
not an exact description of the implementation (which uses $\log(1+P)$ and overlapping
narrow-band filters).

\noindent\emph{Model.} Given a brain state $\theta$, a patch is a segment of a zero-mean
stationary Gaussian process with spectral density $S_\theta$; the state acts on the data only
through $S_\theta$. Asymptotically, the discrete Fourier coefficients at distinct frequencies
are independent circular complex Gaussians, so the periodogram ordinates are $I_k\approx
S_\theta(f_k)E_k$ with $E_k\sim\mathrm{Exp}(1)$, and the phases are uniform and independent of
$\theta$ \cite{brillinger1981}. The Whittle likelihood \cite{whittle1953} depends on the data
only through $\{I_k\}$.

\noindent\emph{Consequences.} (i) The periodogram is sufficient for $\theta$ and the phases are
ancillary (Prop.~\ref{prop:suff}). (ii) If the spectrum is approximately constant within each
of a few bands and the observation is restricted to those bands, the band-averaged periodogram
is the sufficient statistic and maximum-likelihood estimate of the band levels
(Prop.~\ref{prop:mle}). (iii) $\Var[\log P_b]=\psi'(m_b)\approx 1/m_b$ does not depend on the
state, so the log is variance-stabilizing and squared error is an appropriate loss
(Prop.~\ref{prop:var}); for the implemented $\log(1+P)$ the noise variance is about
$m_b^{-1}\bigl(s_b/(1+s_b)\bigr)^2$ and does depend on the band level $s_b$
(Prop.~\ref{prop:var1p}). (iv) For any target, the Bayes risk of predicting a masked patch
from its context splits into the posterior uncertainty about the state plus a term that no
model can reduce (Prop.~\ref{prop:decomp}). For band power the irreducible term is small and
the reducible term is exactly uncertainty about the state; for the waveform the irreducible
term is the patch's own innovation (random amplitude and phase), which scales with the
signal power and, through Parseval's identity, is dominated by low frequencies under a $1/f$
spectrum.

\noindent\emph{Predictions.} When labels describe spectral states, a band-power target should
give representations that are easier to read out linearly than a waveform target, especially
with few labels; reconstructing band-filtered \emph{waveforms} should not help, because it
still requires predicting phase; and the advantage should shrink for tasks that depend on
transient waveform shape or cross-channel phase relations. Sec.~\ref{sec:results} tests these.

\section{Experimental Setup}\label{sec:setup}
\noindent\textbf{Pretraining.} TUEG v2.0.1 at 256\,Hz, 10-s windows, per-recording robust
normalization, with recordings shorter than 5\,min dropped, 60\,s trimmed from each end and
windows exceeding $150\,\mu$V rejected. All models use a random subset of 2{,}100 patients
(859{,}606 windows, 2{,}388\,h; the same subset for every run), AdamW (learning rate
$1.7{\times}10^{-3}$, weight decay $0.05$), 7 warm-up epochs and a cosine schedule defined
over 105 epochs, stopped after epoch 16 for every run (matched steps), effective batch 512 on
four RTX A6000 GPUs. Pretraining seeds 42/43/44 change initialization and data order on the
same subset. Downstream patients are not excluded from pretraining (as in prior work).

\noindent\textbf{Strict linear probe (primary metric)}\label{sec:lpproto}. The encoder is
frozen; the tokens of a window are mean-pooled; features are standardized on the training
split; a class-balanced logistic regression is fitted with $C\in\{0.01,0.1,1\}$ selected on the
validation split by balanced accuracy, and the test split is scored once. For sleep, a 30-s
epoch is the mean of its three 10-s window features. The probe is deterministic, so baselines
with a single released checkpoint need no repetitions; our models are reported over three
pretraining seeds. Label efficiency reuses the selected $C$ on 1--50\% of the training epochs
(five fixed random subsets, identical for all models). Each model receives the same
samples in its own documented input convention: ours at 256\,Hz with per-recording robust
scaling on the 19 standard leads; CBraMod and LaBraM at 200\,Hz, raw $\mu$V$/100$, with the
recording's own channels (LaBraM by its channel names); BIOT with its 95th-percentile scaling
on C3/C4 fed to its C3--A2/C4--A1 channel tokens. This mean-pool protocol does not use the
baselines' downstream heads, which we discuss in Sec.~\ref{sec:limits}.

\noindent\textbf{Fine-tuning (sleep).} Following CBraMod/CSBrain, 30-s epochs are grouped in
sequences of 20; each epoch is encoded (mean of its three 10-s windows), and a shared one-layer
transformer predicts one stage per epoch. Encoder and head are trained with AdamW (learning
rate $10^{-4}$, weight decay $0.05$) for up to 50 epochs with early stopping on validation
balanced accuracy (patience 10).

\noindent\textbf{Downstream data.} Sleep: ISRUC subgroup I \cite{khalighi2016isruc} (subjects
1--80/81--90/91--100) and HMC \cite{alvarez2021hmc} (100/25/26 by sorted subject ID), the
splits of CBraMod and CSBrain. Beyond sleep we use CBraMod's own preprocessing scripts to
produce samples and subject-disjoint splits: Mental Arithmetic \cite{zyma2019mentalarith}
(stress, 2-class, 5-s windows), Mumtaz \cite{mumtaz2016mdd} (depression, 2-class, 5-s),
FACED \cite{chen2023faced} (emotion, 9-class, 10-s; rebuilt from the public raw recordings),
PhysioNet-MI \cite{schalk2004bci2000} (motor imagery, 4-class, 4-s), and SEED-VIG
\cite{zheng2017seedvig} (vigilance regression, 8-s; we use the split stated in the CBraMod
paper, see Appendix~\ref{app:data}). Our sample counts match CBraMod's (e.g.\ 1{,}707 for MA,
7{,}143 for Mumtaz).

\section{Results}\label{sec:results}

\begin{table*}[t]\centering\small
\setlength{\tabcolsep}{4pt}
\caption{\textbf{Pretext-target study} at matched backbone, data ($2{,}388$\,h), steps
(16 epochs) and launcher; only the target (or, last row, the mask axis) changes.
Filterbank tokenizer; strict linear probe, mean${\pm}$std over three pretraining seeds.
``1\%'': balanced accuracy with $1\%$ of the labelled training epochs (mean of five subsets).
$\Delta$: difference to the band-power row in balanced-accuracy points.}
\label{tab:targets}
\begin{tabular}{l cc cc cc}
\toprule
& \multicolumn{2}{c}{BA (100\% labels)} & \multicolumn{2}{c}{BA (1\% labels)} & \multicolumn{2}{c}{$\Delta$ BA (pts)}\\
\cmidrule(lr){2-3}\cmidrule(lr){4-5}\cmidrule(lr){6-7}
Pretext target & ISRUC & HMC & ISRUC & HMC & ISRUC & HMC\\
\midrule
\textbf{Band power} (fixed narrow-band filters) & \val{0.7863}{0.0039} & \val{0.7394}{0.0025} & \val{0.758}{0.002} & \val{0.674}{0.009} & -- & -- \\
\quad band power, leak-free (per-patch filtering) & \val{0.7906}{0.0018} & \val{0.7398}{0.0044} & \val{0.753}{0.004} & \val{0.675}{0.012} & +0.4 & +0.0 \\
\quad band power, standard bands (per-patch FFT) & \val{0.7827}{0.0021} & \val{0.7339}{0.0021} & \val{0.725}{0.009} & \val{0.638}{0.004} & -0.4 & -0.6 \\
\midrule
Raw-waveform reconstruction & \val{0.7677}{0.0028} & \val{0.7180}{0.0039} & \val{0.711}{0.026} & \val{0.623}{0.018} & -1.9 & -2.1 \\
Band-waveform reconstruction & \val{0.7682}{0.0061} & \val{0.7119}{0.0024} & \val{0.697}{0.019} & \val{0.613}{0.018} & -1.8 & -2.7 \\
Band power, frequency-axis masking & \val{0.7639}{0.0058} & \val{0.7051}{0.0019} & \val{0.690}{0.018} & \val{0.639}{0.001} & -2.2 & -3.4 \\
\bottomrule
\end{tabular}
\end{table*}

\begin{table*}[t]\centering\small
\setlength{\tabcolsep}{3.5pt}
\caption{\textbf{Tokenizer $\times$ target.} Band power (BP) vs.\ raw-waveform
reconstruction (Raw) under both tokenizers; everything else matched. Strict linear probe,
mean${\pm}$std over three pretraining seeds. $\Delta$: paired (same-seed) difference in
balanced-accuracy points; ``wins'': seeds on which BP $>$ Raw.}
\label{tab:twobytwo}
\begin{tabular}{l cccc cccc}
\toprule
& \multicolumn{4}{c}{ISRUC} & \multicolumn{4}{c}{HMC}\\
\cmidrule(lr){2-5}\cmidrule(lr){6-9}
Tokenizer & BP & Raw & $\Delta$ & wins & BP & Raw & $\Delta$ & wins\\
\midrule
Linear (CBraMod-style patch embedding) & \val{0.7916}{0.0081} & \val{0.7755}{0.0034} & +1.6\,{\scriptsize$\pm$0.7} & 3/3 & \val{0.7425}{0.0045} & \val{0.7141}{0.0062} & +2.8\,{\scriptsize$\pm$1.1} & 3/3 \\
Filterbank & \val{0.7863}{0.0039} & \val{0.7677}{0.0028} & +1.9\,{\scriptsize$\pm$0.3} & 3/3 & \val{0.7394}{0.0025} & \val{0.7180}{0.0039} & +2.1\,{\scriptsize$\pm$0.1} & 3/3 \\
\midrule
Linear, 1\% labels & \val{0.760}{0.002} & \val{0.708}{0.010} & +5.2\,{\scriptsize$\pm$1.2} & 3/3 & \val{0.685}{0.001} & \val{0.613}{0.015} & +7.3\,{\scriptsize$\pm$1.6} & 3/3 \\
Filterbank, 1\% labels & \val{0.758}{0.002} & \val{0.711}{0.026} & +4.7\,{\scriptsize$\pm$2.5} & 3/3 & \val{0.674}{0.009} & \val{0.623}{0.018} & +5.1\,{\scriptsize$\pm$0.9} & 3/3 \\
\bottomrule
\end{tabular}
\end{table*}

\begin{figure}[t]\centering
\includegraphics[width=\columnwidth]{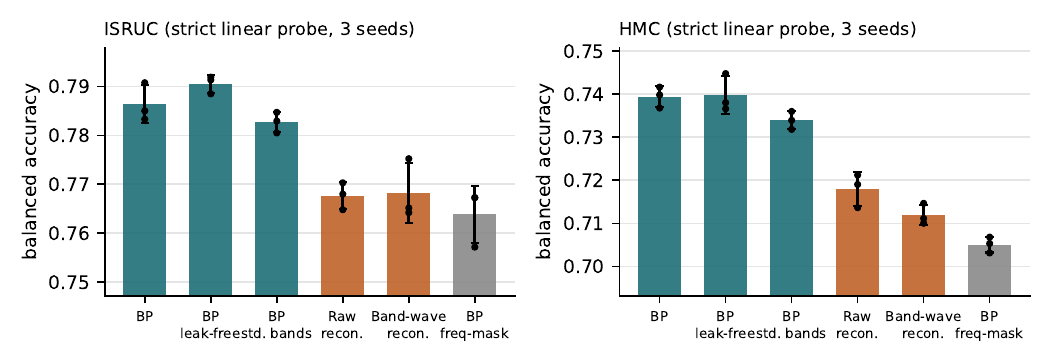}
\caption{\textbf{Target study} (Table~\ref{tab:targets}): bars are means over three
pretraining seeds, dots individual seeds. Blue: band-power targets; orange: waveform
reconstruction; grey: band power with frequency-axis masking.}
\label{fig:targets}
\end{figure}

\subsection{The target effect exceeds the tokenizer effect}
Table~\ref{tab:targets} changes only the pretext target. The three band-power variants---the
main narrow-band target, a leak-free variant that filters each patch in isolation, and exact
standard bands via per-patch FFT---lie close together (ISRUC $0.7827$--$0.7906$, HMC
$0.7339$--$0.7398$), although the FFT variant is $0.6$ points below the main target on HMC,
more than one seed standard deviation. Both reconstruction targets are 1.8--2.8 balanced-accuracy
points lower than the main target (raw $0.7677$/$0.7180$, band-waveform $0.7682$/$0.7119$),
in the same direction on all three seeds, and the gap grows with fewer labels (at $1\%$: $0.758$ vs.\ $0.711$
on ISRUC, $0.674$ vs.\ $0.623$ on HMC). Band-waveform reconstruction is as poor as raw
reconstruction. This is consistent with the prediction that discarding phase and the patch's
innovation is what helps, rather than moving to the frequency domain, but it is not a clean
attribution: the band-waveform target also has a $5{\times}200$-dimensional output and a
different loss scale. Masking frequency bands instead of time patches is worst, supporting
masking in time and predicting the full spectral profile of a patch from its spatio-temporal
context. The leak-free and per-patch-FFT variants, which use no samples outside the patch,
keep the advantage over reconstruction; leakage through the 65-tap filters is therefore not
necessary for the effect, although these controls do not show whether the main model makes
use of it.

Table~\ref{tab:twobytwo} crosses target and tokenizer. The target effect is the same with a
plain linear patch embedding ($+1.6$/$+2.8$ points on ISRUC/HMC) and with the filterbank
tokenizer ($+1.9$/$+2.1$), positive on all 12 seed--dataset pairs, and $+4.7$ to $+7.3$
points with $1\%$ of labels. In these experiments the tokenizer effect is smaller (linear
$0.7916$/$0.7425$ vs.\ filterbank $0.7863$/$0.7394$), so we use the simpler linear tokenizer as
the main model.
Along pretraining (Table~\ref{tab:compute}), band power at epoch 4 already exceeds raw
reconstruction at epoch 16 on both datasets.

\begin{table}[t]\centering\small
\caption{\textbf{Along pretraining} (filterbank tokenizer, seed 42, strict linear probe BA).
Band power at epoch 4 already exceeds raw-waveform reconstruction at epoch 16.}
\label{tab:compute}
\begin{tabular}{c cc cc}
\toprule
& \multicolumn{2}{c}{ISRUC} & \multicolumn{2}{c}{HMC}\\
\cmidrule(lr){2-3}\cmidrule(lr){4-5}
Epoch & BP & Raw & BP & Raw\\
\midrule
4 & 0.7719 & 0.7394 & 0.7418 & 0.6714 \\
8 & 0.7862 & 0.7630 & 0.7339 & 0.7026 \\
16 & 0.7833 & 0.7680 & 0.7398 & 0.7190 \\
\bottomrule
\end{tabular}
\end{table}

\begin{table*}[t]\centering\small
\setlength{\tabcolsep}{4pt}
\caption{\textbf{Strict linear probe on sleep staging.} Frozen encoder, mean-pooled
features, logistic regression with $C$ selected on the validation split, test scored once
(deterministic); a 30-s epoch is the mean of its three 10-s window features. Every model we
ran sees the identical epochs and splits, each in its own input convention
(Sec.~\ref{sec:lpproto}). SpecBraM rows: mean${\pm}$std over three pretraining seeds.
$^\dagger$Quoted from the REVE paper (its own linear-probing protocol). Handcrafted rows:
features averaged over the three 10-s windows of each epoch (as for every encoder; better than
whole-epoch features), same samples and probe;
$^\ast$1-Hz log PSD (0.5--45\,Hz), relative power in six bands, Hjorth parameters, spectral
entropy, 95\% spectral edge, line length and kurtosis per channel.}
\label{tab:lpmain}
\begin{tabular}{l c ccc ccc}
\toprule
& & \multicolumn{3}{c}{ISRUC} & \multicolumn{3}{c}{HMC}\\
\cmidrule(lr){3-5}\cmidrule(lr){6-8}
Model & Pretrain (h) & BA & $\kappa$ & F1 & BA & $\kappa$ & F1\\
\midrule
Random init (our arch.) & 0 & 0.6884 & 0.6111 & 0.6957 & 0.6256 & 0.5274 & 0.6332 \\
BIOT (six-dataset) & mixed & 0.7358 & 0.6658 & 0.7403 & 0.6578 & 0.5623 & 0.6610 \\
BIOT (SHHS+PREST) & mixed & 0.7430 & 0.6683 & 0.7411 & 0.6105 & 0.5042 & 0.6175 \\
CBraMod & ${\sim}9$k & 0.7278 & 0.6581 & 0.7343 & 0.6190 & 0.5131 & 0.6193 \\
LaBraM-Base & ${\sim}2.5$k & 0.7523 & 0.6857 & 0.7543 & 0.6323 & 0.5303 & 0.6357 \\
REVE-Base$^\dagger$ & ${\sim}60$k & 0.697 & -- & -- & 0.647 & -- & -- \\
REVE-Large$^\dagger$ & ${\sim}60$k & 0.758 & -- & -- & 0.710 & -- & -- \\
\midrule
Handcrafted: 5-band log/relative power & -- & 0.7418 & 0.6716 & 0.7473 & 0.6825 & 0.5900 & 0.6776 \\
Handcrafted: rich spectral set$^\ast$ & -- & 0.7636 & 0.7003 & 0.7730 & 0.7227 & 0.6484 & 0.7263 \\
\midrule
SpecBraM, filterbank tokenizer & $2{,}388$ & \val{0.7863}{0.0039} & \val{0.7338}{0.0046} & \val{0.7942}{0.0034} & \val{0.7394}{0.0025} & \val{0.6701}{0.0023} & \val{0.7457}{0.0019} \\
\textbf{SpecBraM} & $2{,}388$ & \val{0.7916}{0.0081} & \val{0.7375}{0.0099} & \val{0.7974}{0.0065} & \val{0.7425}{0.0045} & \val{0.6765}{0.0050} & \val{0.7495}{0.0034} \\
\bottomrule
\end{tabular}
\end{table*}

\subsection{Frozen features versus EEG foundation models}
Under the same strict linear probe (Table~\ref{tab:lpmain}), SpecBraM reaches
$0.7916{\pm}0.0081$ on ISRUC and $0.7425{\pm}0.0045$ on HMC, against $0.7523$/$0.6323$ for
LaBraM, $0.7278$/$0.6190$ for CBraMod and at most $0.7430$/$0.6578$ for BIOT; REVE-Large
reports $0.758$/$0.710$ under its own probing protocol. All baselines are above the
random-initialization floor of our architecture ($0.6884$/$0.6256$), which is a basic sanity
check of our adapters; it does not exclude that each model's input adaptation and mean-pooled
readout affect the comparison (Sec.~\ref{sec:limits}). Label efficiency
(Table~\ref{tab:labeleff}, Fig.~\ref{fig:labeleff}) is the clearest view: with $1\%$ of the
training labels SpecBraM reaches $0.760$/$0.685$, above the full-label probes of all external
pretrained models we test (best $0.752$ on ISRUC, $0.658$ on HMC). It does not exceed every
model: the same-architecture raw-reconstruction model reaches $0.775$/$0.714$ with all labels,
and needs about 10--25\% of the labels to match what SpecBraM achieves with 1\%.

\noindent\textbf{Direct spectral features.} A natural question is what self-supervision adds over
simply computing spectral features. We probe two handcrafted feature sets on the same epochs
with the same probe, averaging over the three 10-s windows as for every encoder (this was
better than whole-epoch features): five-band log and relative power, and a rich set per
channel (1-Hz log power spectrum from 0.5 to 45\,Hz, relative power in six bands, Hjorth
parameters, spectral entropy, 95\% spectral edge, line length and kurtosis). The rich set
reaches $0.7636$/$0.7227$ (Table~\ref{tab:lpmain})---above every external foundation model under
this probe, and above the same-architecture raw-reconstruction model on HMC ($0.7141$). SpecBraM
exceeds it by $2.8$/$2.0$ points with all labels, but only by about one point with $1\%$ of the
labels ($0.760$/$0.685$ vs.\ $0.747$/$0.674$; Table~\ref{tab:labeleff}), and with all labels the
rich features exceed SpecBraM's 1\%-label result. Self-supervision with a spectral target thus
adds a modest but consistent gain over direct spectral features for sleep staging, whereas
waveform reconstruction does not reliably reach them. All probes share the same grid
$C\in\{0.01,0.1,1\}$; the handcrafted probes select its lower edge, so a wider grid might
raise them slightly.

\begin{table}[t]\centering\small
\setlength{\tabcolsep}{3pt}
\caption{\textbf{Label efficiency} (strict linear probe, balanced accuracy vs.\ fraction of
labelled training epochs; ${<}100\%$: mean of five fixed random subsets shared by all
models). Our rows: mean over three pretraining seeds.}
\label{tab:labeleff}
\resizebox{\columnwidth}{!}{%
\begin{tabular}{l cccccc}
\toprule
Model & 1\% & 5\% & 10\% & 25\% & 50\% & 100\%\\
\midrule
\multicolumn{7}{l}{\emph{ISRUC}}\\
Random init (our arch.) & 0.542 & 0.615 & 0.644 & 0.672 & 0.684 & 0.689 \\
BIOT & 0.655 & 0.703 & 0.716 & 0.728 & 0.734 & 0.736 \\
CBraMod & 0.660 & 0.709 & 0.717 & 0.726 & 0.727 & 0.729 \\
LaBraM-Base & 0.680 & 0.724 & 0.730 & 0.741 & 0.746 & 0.752 \\
Handcrafted, 5-band & 0.707 & 0.726 & 0.731 & 0.739 & 0.741 & 0.742 \\
Handcrafted, rich & 0.747 & 0.758 & 0.763 & 0.764 & 0.764 & 0.763 \\
Raw-waveform recon.\ (same arch.) & 0.708 & 0.742 & 0.755 & 0.767 & 0.773 & 0.775 \\
\textbf{SpecBraM} & 0.760 & 0.781 & 0.786 & 0.791 & 0.791 & 0.792 \\
\multicolumn{7}{l}{\emph{HMC}}\\
Random init (our arch.) & 0.515 & 0.580 & 0.598 & 0.612 & 0.619 & 0.626 \\
BIOT & 0.582 & 0.625 & 0.638 & 0.648 & 0.655 & 0.658 \\
CBraMod & 0.539 & 0.590 & 0.602 & 0.613 & 0.616 & 0.620 \\
LaBraM-Base & 0.574 & 0.616 & 0.621 & 0.630 & 0.631 & 0.633 \\
Handcrafted, 5-band & 0.624 & 0.656 & 0.665 & 0.675 & 0.678 & 0.682 \\
Handcrafted, rich & 0.674 & 0.702 & 0.709 & 0.717 & 0.720 & 0.723 \\
Raw-waveform recon.\ (same arch.) & 0.613 & 0.664 & 0.685 & 0.702 & 0.708 & 0.714 \\
\textbf{SpecBraM} & 0.685 & 0.719 & 0.730 & 0.739 & 0.741 & 0.742 \\
\bottomrule
\end{tabular}}
\end{table}

\begin{figure}[t]\centering
\includegraphics[width=\columnwidth]{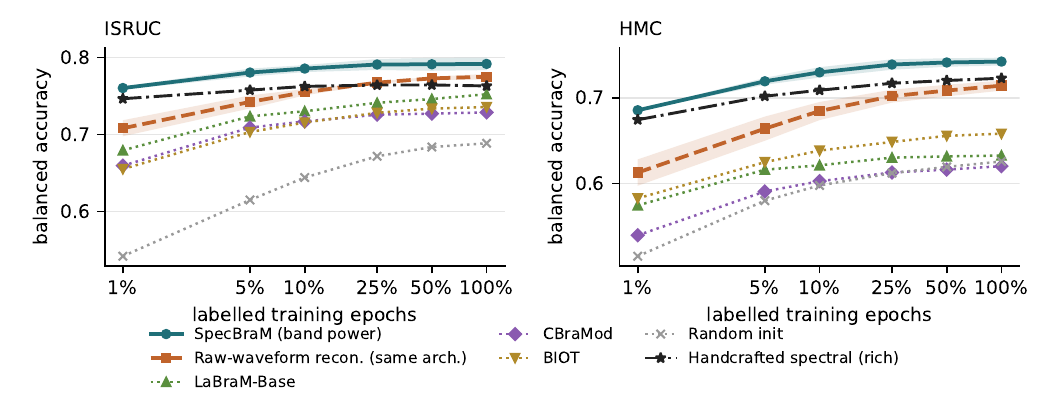}
\caption{\textbf{Label efficiency} under the strict linear probe. Our curves: mean over
three pretraining seeds with ${\pm}1$ std band.}
\label{fig:labeleff}
\end{figure}

\begin{table*}[t]\centering\footnotesize
\setlength{\tabcolsep}{3pt}
\caption{\textbf{Sleep staging, fine-tuned} (5-class, 30-s epochs in sequences of 20, one
shared 1-layer sequence transformer; identical subject-disjoint splits as
\citet{csbrain}). Baselines as reported by \citet{csbrain} (REVE: its paper, ISRUC only).
Our rows: mean${\pm}$std over three \emph{pretraining} seeds (one fine-tuning run each), so
the spread includes pretraining variance, unlike the baselines' fine-tuning-seed spread.}
\label{tab:ft30}
\resizebox{\textwidth}{!}{%
\begin{tabular}{l c ccc ccc}
\toprule
& & \multicolumn{3}{c}{ISRUC} & \multicolumn{3}{c}{HMC}\\
\cmidrule(lr){3-5}\cmidrule(lr){6-8}
Method & Pretrain (h) & BA & $\kappa$ & Weighted F1 & BA & $\kappa$ & Weighted F1\\
\midrule
\multicolumn{8}{l}{\emph{Task-specific models}}\\
EEGNet~\cite{lawhern2018eegnet} & -- & \val{0.7154}{.0121} & \val{0.7040}{.0173} & \val{0.7513}{.0124} & \val{0.6534}{.0122} & \val{0.5886}{.0201} & \val{0.6536}{.0168} \\
EEGConformer~\cite{song2023conformer} & -- & \val{0.7400}{.0133} & \val{0.7143}{.0162} & \val{0.7634}{.0151} & \val{0.7149}{.0086} & \val{0.6432}{.0055} & \val{0.7080}{.0039} \\
SPaRCNet~\cite{jing2023sparcnet} & -- & \val{0.7487}{.0075} & \val{0.7097}{.0132} & \val{0.7624}{.0092} & \val{0.4756}{.1109} & \val{0.3147}{.1315} & \val{0.4108}{.1310} \\
ContraWR~\cite{yang2023contrawr} & -- & \val{0.7402}{.0126} & \val{0.7178}{.0156} & \val{0.7610}{.0137} & \val{0.4242}{.0541} & \val{0.2340}{.0554} & \val{0.2987}{.0288} \\
CNN-Transformer~\cite{peh2022cnntransformer} & -- & \val{0.7363}{.0087} & \val{0.7129}{.0121} & \val{0.7719}{.0153} & \val{0.6573}{.0141} & \val{0.5961}{.0105} & \val{0.6896}{.0065} \\
FFCL~\cite{li2022ffcl} & -- & \val{0.7277}{.0182} & \val{0.7016}{.0292} & \val{0.7614}{.0197} & \val{0.4427}{.0702} & \val{0.2542}{.0654} & \val{0.2902}{.0485} \\
ST-Transformer~\cite{song2021sttransformer} & -- & \val{0.7381}{.0205} & \val{0.7013}{.0352} & \val{0.7681}{.0175} & \val{0.2559}{.0141} & \val{0.0503}{.0183} & \val{0.1428}{.0122} \\
\midrule
\multicolumn{8}{l}{\emph{Foundation models}}\\
BIOT~\cite{yang2023biot} & mixed & \val{0.7527}{.0121} & \val{0.7291}{.0230} & \val{0.7790}{.0146} & \val{0.6862}{.0041} & \val{0.6295}{.0113} & \val{0.7091}{.0147} \\
LaBraM-Base~\cite{jiang2024labram} & ${\sim}2.5$k & \val{0.7633}{.0102} & \val{0.7231}{.0182} & \val{0.7810}{.0133} & \val{0.7277}{.0101} & \val{0.6813}{.0053} & \val{0.7454}{.0027} \\
CBraMod~\cite{wang2025cbramod} & ${\sim}9$k & \val{0.7865}{.0110} & \val{0.7442}{.0152} & \val{0.8011}{.0099} & \val{0.7269}{.0041} & \val{0.6685}{.0104} & \val{0.7395}{.0089} \\
CSBrain~\cite{csbrain} & ${\sim}9$k & \val{0.7925}{.0030} & \val{0.7406}{.0102} & \val{0.7990}{.0091} & \val{0.7345}{.0047} & \val{0.6818}{.0046} & \val{0.7506}{.0042} \\
REVE~\cite{reve2025} & ${\sim}60$k & 0.7819 & 0.7500 & 0.8005 & -- & -- & -- \\
\midrule
Raw-waveform recon.\ (same arch.) & $2{,}388$ & \val{0.8008}{0.0066} & \val{0.7465}{0.0079} & \val{0.8043}{0.0074} & \val{0.7550}{0.0114} & \val{0.6978}{0.0148} & \val{0.7661}{0.0140} \\
\textbf{SpecBraM} & $2{,}388$ & \val{0.8107}{0.0048} & \val{0.7559}{0.0121} & \val{0.8102}{0.0103} & \val{0.7669}{0.0059} & \val{0.7177}{0.0046} & \val{0.7829}{0.0023} \\
\bottomrule
\end{tabular}}
\end{table*}

\subsection{Fine-tuning on sleep staging}
With full fine-tuning on 30-s epochs (Table~\ref{tab:ft30}), SpecBraM reaches
$0.8107{\pm}0.0048$ balanced accuracy on ISRUC and $0.7669{\pm}0.0059$ on HMC, above the best
published results (CSBrain $0.7925$/$0.7345$), and the highest Cohen's $\kappa$ on both
datasets. SpecBraM uses about a quarter of their pretraining hours, but the models also differ
in architecture and training recipe, so we do not read this comparison as evidence of data
efficiency on its own. Our standard deviations are across
\emph{pretraining} seeds and therefore include more sources of variance than the baselines'
fine-tuning seeds. After fine-tuning the gap to the matched raw-reconstruction model narrows
to about one point ($0.8008$/$0.7550$): full fine-tuning partly compensates for a weaker
initialization, and the target matters most for frozen and low-label use, as the analysis
predicts.

\begin{table*}[t]\centering\small
\setlength{\tabcolsep}{3.5pt}
\caption{\textbf{Strict linear probe beyond sleep: binary tasks.} Samples, channels and
subject-disjoint splits are produced by CBraMod's own preprocessing scripts, so every model
is probed on identical samples. Mental Arithmetic (MA) has only 4 test subjects.
$^\dagger$CBraMod's \emph{fine-tuned} result from its paper, for reference.}
\label{tab:taskbin}
\begin{tabular}{l ccc ccc}
\toprule
& \multicolumn{3}{c}{MA (stress, 2-class)} & \multicolumn{3}{c}{Mumtaz (MDD, 2-class)}\\
\cmidrule(lr){2-4}\cmidrule(lr){5-7}
Model & BA & AUC-PR & AUROC & BA & AUC-PR & AUROC\\
\midrule
\textbf{SpecBraM} & \val{0.752}{0.014} & \val{0.665}{0.034} & \val{0.840}{0.020} & \val{0.935}{0.003} & \val{0.991}{0.005} & \val{0.990}{0.005} \\
Raw-waveform recon.\ (same arch.) & \val{0.666}{0.022} & \val{0.606}{0.026} & \val{0.773}{0.024} & \val{0.935}{0.020} & \val{0.983}{0.009} & \val{0.983}{0.008} \\
SpecBraM, filterbank tokenizer & \val{0.693}{0.047} & \val{0.655}{0.064} & \val{0.816}{0.045} & \val{0.941}{0.023} & \val{0.990}{0.008} & \val{0.990}{0.008} \\
Raw recon., filterbank tokenizer & \val{0.730}{0.020} & \val{0.699}{0.041} & \val{0.810}{0.026} & \val{0.922}{0.012} & \val{0.978}{0.007} & \val{0.979}{0.007} \\
\midrule
CBraMod (official) & 0.653 & 0.515 & 0.708 & 0.893 & 0.973 & 0.973 \\
LaBraM-Base (official) & 0.681 & 0.571 & 0.734 & 0.876 & 0.960 & 0.963 \\
Random init (our arch.) & 0.538 & 0.276 & 0.577 & 0.871 & 0.947 & 0.945 \\
\midrule
\emph{CBraMod, fine-tuned}$^\dagger$ & 0.726 & 0.627 & 0.791 & 0.956 & 0.992 & 0.992 \\
\bottomrule
\end{tabular}
\end{table*}

\begin{table*}[t]\centering\small
\setlength{\tabcolsep}{3.5pt}
\caption{\textbf{Strict linear probe beyond sleep: multi-class tasks} (protocol as
Table~\ref{tab:taskbin}). FACED: 9-class emotion (chance $0.111$); PhysioNet-MI: 4-class
motor imagery (chance $0.25$). $^\ddagger$CBraMod with a frozen backbone and a trained
classifier, from its paper (Table~18). $^\dagger$CBraMod fine-tuned, from its paper.}
\label{tab:taskmulti}
\begin{tabular}{l ccc ccc}
\toprule
& \multicolumn{3}{c}{FACED (emotion, 9-class)} & \multicolumn{3}{c}{PhysioNet-MI (4-class)}\\
\cmidrule(lr){2-4}\cmidrule(lr){5-7}
Model & BA & $\kappa$ & F1 & BA & $\kappa$ & F1\\
\midrule
\textbf{SpecBraM} & \val{0.241}{0.011} & \val{0.146}{0.012} & \val{0.240}{0.011} & \val{0.388}{0.013} & \val{0.184}{0.017} & \val{0.386}{0.013} \\
Raw-waveform recon.\ (same arch.) & \val{0.216}{0.005} & \val{0.118}{0.007} & \val{0.213}{0.007} & \val{0.397}{0.008} & \val{0.196}{0.011} & \val{0.397}{0.009} \\
SpecBraM, filterbank tokenizer & \val{0.226}{0.002} & \val{0.130}{0.002} & \val{0.226}{0.002} & \val{0.383}{0.015} & \val{0.178}{0.021} & \val{0.382}{0.016} \\
Raw recon., filterbank tokenizer & \val{0.204}{0.016} & \val{0.105}{0.017} & \val{0.202}{0.014} & \val{0.380}{0.016} & \val{0.173}{0.021} & \val{0.380}{0.016} \\
\midrule
CBraMod (official) & 0.171 & 0.069 & 0.167 & 0.312 & 0.083 & 0.310 \\
LaBraM-Base (official) & 0.167 & 0.064 & 0.163 & 0.299 & 0.065 & 0.299 \\
Random init (our arch.) & 0.129 & 0.021 & 0.130 & 0.267 & 0.023 & 0.267 \\
\midrule
\emph{CBraMod, frozen backbone}$^\ddagger$ & 0.315 & 0.258 & 0.308 & 0.385 & 0.298 & 0.395 \\
\emph{CBraMod, fine-tuned}$^\dagger$ & 0.551 & 0.504 & 0.562 & 0.642 & 0.522 & 0.643 \\
\bottomrule
\end{tabular}
\end{table*}

\subsection{Beyond sleep}
On four further tasks probed on CBraMod's own samples and splits
(Tables~\ref{tab:taskbin} and~\ref{tab:taskmulti}), SpecBraM's frozen features outperform the
official CBraMod and LaBraM weights under the same probe on every task: Mental Arithmetic
$0.752$ vs.\ $0.653$/$0.681$ balanced accuracy (and above CBraMod's own fine-tuned $0.726$);
Mumtaz $0.935$ vs.\ $0.893$/$0.876$; FACED $0.241$ vs.\ $0.171$/$0.167$; PhysioNet-MI $0.388$
vs.\ $0.312$/$0.299$. On FACED and PhysioNet-MI, CBraMod's paper reports $0.315$ and $0.385$
with a frozen backbone but its own downstream classifier on all tokens; our mean-pooled probe
of the same weights gives $0.171$ and $0.312$, so part of the gap there reflects the probing
head rather than the representation.

The \emph{target} comparison beyond sleep is mixed. Band power beats raw reconstruction on
FACED under both tokenizers ($+2.4$/$+2.2$ points) and on MA with the linear tokenizer
($+8.6$), but not with the filterbank tokenizer ($-3.7$; MA has only four test subjects); it
ties on Mumtaz (BA equal, AUC slightly higher) and on PhysioNet-MI ($-0.9$/$+0.3$). Motor
imagery does carry spectral information---event-related $\mu/\beta$ changes---but it is
spatially localized and lateralized, and a per-channel band-energy target over 0.78-s patches
may not emphasize it. The benefit therefore does not extend to every task with spectral cues;
the target also has to match the spatial and temporal scale of the label.

\begin{table*}[t]\centering\small
\setlength{\tabcolsep}{4pt}
\caption{\textbf{SEED-VIG vigilance regression (PERCLOS)} on the split stated in the CBraMod
paper (experiments 1--15/16--19/20--23 in subject order). Linear probe = frozen encoder +
ridge regression (regularization on validation $R^2$); FT = full fine-tuning (MSE, epoch
selected on validation $R^2$). Input band 0.5--45\,Hz (standard) or 4--45\,Hz (high-passed;
the cutoff the validation split selects for both linear-tokenizer targets). Our rows: mean${\pm}$std over three pretraining seeds.
$^\dagger$From the CBraMod paper.}
\label{tab:seedvig}
\begin{tabular}{l cc cc cc}
\toprule
& \multicolumn{2}{c}{LP, 0.5--45\,Hz} & \multicolumn{2}{c}{LP, 4--45\,Hz} & \multicolumn{2}{c}{FT, 0.5--45\,Hz}\\
\cmidrule(lr){2-3}\cmidrule(lr){4-5}\cmidrule(lr){6-7}
Model & Pearson & $R^2$ & Pearson & $R^2$ & Pearson & $R^2$\\
\midrule
\textbf{SpecBraM} & \val{0.565}{0.012} & \val{0.166}{0.035} & \val{0.619}{0.006} & \val{0.250}{0.012} & \val{0.587}{0.033} & \val{0.260}{0.083} \\
Raw-waveform recon.\ (same arch.) & \val{0.548}{0.031} & \val{0.209}{0.038} & \val{0.572}{0.022} & \val{0.261}{0.051} & \val{0.626}{0.039} & \val{0.357}{0.049} \\
SpecBraM, filterbank tokenizer & \val{0.540}{0.021} & \val{0.176}{0.035} & \val{0.611}{0.007} & \val{0.257}{0.021} & -- & -- \\
Raw recon., filterbank tokenizer & \val{0.564}{0.005} & \val{0.221}{0.029} & \val{0.549}{0.011} & \val{0.213}{0.017} & -- & -- \\
\midrule
LaBraM-Base (official) & 0.532 & 0.162 & -- & -- & -- & -- \\
CBraMod (official) & 0.343 & -0.017 & -- & -- & -- & -- \\
Random init & 0.442 & 0.047 & 0.481 & 0.143 & -- & -- \\
\midrule
\emph{CBraMod, fine-tuned}$^\dagger$ & -- & -- & -- & -- & 0.646 & 0.337 \\
\bottomrule
\end{tabular}
\end{table*}

\subsection{Where the benefit stops: vigilance}\label{sec:boundary}
SEED-VIG (Table~\ref{tab:seedvig}) regresses PERCLOS, an eye-closure measure. With the
standard 0.5--45\,Hz input the band-power and raw targets tie in Pearson correlation
($0.565$ vs.\ $0.548$), raw reconstruction is better in $R^2$ ($0.209$ vs.\ $0.166$), and after
fine-tuning raw reconstruction is clearly better ($0.626$/$0.357$ vs.\ $0.587$/$0.260$). A
simple explanation---that waveform targets help because they preserve blink shapes---is not
supported by a high-pass experiment: removing activity below 4\,Hz \emph{improves} the
band-power models by 5--7 points of correlation ($0.565\to0.619$ linear, $0.540\to0.611$
filterbank) while leaving the raw-reconstruction models unchanged or worse ($0.548\to0.572$,
$0.564\to0.549$); the validation split independently selects the 4\,Hz cutoff for both
linear-tokenizer models. The high-pass changes the input to the frozen encoder and the
spectral content of each patch at the same time, so it does not identify a mechanism. One
hypothesis is that ocular artifacts, one to two orders of magnitude larger than cortical
activity, dominate the lowest target bands and the corresponding dimensions of the
representation, masking drowsiness-related alpha and theta changes; we have not tested it
directly (e.g.\ by high-passing only the target during pretraining). The case shows that a
spectral-energy target helps only when the energy it measures matches the label, which is not
guaranteed in the presence of large artifacts.

\section{Limitations}\label{sec:limits}
\noindent\emph{The implemented target.} The main target uses five overlapping narrow-band
filters and $\log(1+P)$. The former does not represent activity below 2\,Hz; the latter is
nearly linear for the small energies of the high bands, so the variance-stabilizing argument
of Sec.~\ref{sec:theory} holds only partially there. The standard-band FFT variant keeps the
advantage over reconstruction but is slightly lower on HMC; per-band standardized log power
(as in FAME \cite{yu2026fame}) is a natural next step.

\noindent\emph{Probing protocol.} Our probe mean-pools tokens and fits a linear classifier,
which is strict and identical across models but does not use the baselines' own downstream
heads; on FACED and PhysioNet-MI CBraMod's frozen-backbone result with its own head is higher
than our probe of it. REVE numbers are quoted from its paper.

\noindent\emph{Scope.} The advantage of the target is clearest for sleep staging; it does
not extend to every task with spectral cues (motor imagery, depression screening, vigilance).
Several cohorts are small (MA has four test subjects). Downstream recordings may overlap TUEG
for all models; an abnormality-detection (TUAB) evaluation is in progress. Pretraining uses a
single 2{,}388-hour budget and one model size; comparisons with external models that differ in
data, duration and architecture do not establish data efficiency in general.

\noindent\emph{Theory.} Sec.~\ref{sec:theory} assumes stationarity, Gaussianity and a state
that acts only through the spectrum. Real EEG contains transients (K-complexes, spikes,
blinks), heavy tails and cross-channel phase relations; the analysis motivates the target and
yields predictions but is not a proof that band power is optimal.

\section{Conclusion}
With the same backbone, data and training steps, predicting fixed narrow-band spectral
energies of masked patches yields sleep representations that are easier to read out linearly
than those learned by waveform reconstruction, on every seed and with either tokenizer. The
advantage is largest with few labels and appears early in pretraining, but not every task with
spectral information benefits: the target also has to match the spatial scale, time scale and
artifact exposure of the label. We hope the controlled comparisons and the failure cases help
choose pretext targets for physiological signals more deliberately.

\appendix

\section{Idealized Analysis: Statements and Proofs}\label{app:theory}
Let $\bm{x}=(x_0,\dots,x_{N-1})$ be one channel of a patch, $f_k=kf_s/N$, and consider
$1\le k\le K=\lfloor(N-1)/2\rfloor$. Define $X_k=\sum_t x_te^{-i2\pi kt/N}$,
$I_k=|X_k|^2/N$, $\varphi_k=\arg X_k$.

\begin{assumption}[stationary Gaussian; state acts through the spectrum]\label{as:gauss}
Given the state $\theta$, $\bm{x}$ is a segment of a zero-mean weakly stationary Gaussian
process with positive spectral density $S_\theta$, continuous except at finitely many points
and continuous at the $f_k$ used.
\end{assumption}

\begin{lemma}[Brillinger 1981, Thm.~4.4.1]\label{lem:dft}
Under Assumption~\ref{as:gauss}, as $N\to\infty$, for a fixed set of frequencies,
$X_k/\sqrt{NS_\theta(f_k)}$ converges in distribution to independent standard circular complex
Gaussians. Hence approximately $I_k\approx S_\theta(f_k)E_k$ with $E_k$ i.i.d.\
$\mathrm{Exp}(1)$, and $\varphi_k$ i.i.d.\ uniform on $[0,2\pi)$, independent of $\{I_k\}$,
with a distribution that does not depend on $\theta$.
\end{lemma}
\noindent The Whittle approximation \cite{whittle1953} follows:
\begin{equation}\label{eq:whittle}
-\log p(\bm{x}\mid\theta)\approx\sum_{k=1}^{K}\Bigl[\log S_\theta(f_k)+\frac{I_k}{S_\theta(f_k)}\Bigr]+\text{const}.
\end{equation}

\begin{proposition}[phase is ancillary]\label{prop:suff}
Under Assumption~\ref{as:gauss} and \eqref{eq:whittle}, $(I_1,\dots,I_K)$ is sufficient for
$\theta$, and the distribution of $(\varphi_1,\dots,\varphi_K)$ does not depend on $\theta$.
\end{proposition}
\begin{proof}
The data enter \eqref{eq:whittle} only through $\{I_k\}$ (Fisher--Neyman factorization); the
phase distribution is given by Lemma~\ref{lem:dft}.
\end{proof}
\noindent Not only the phase is unpredictable: each $X_k$ is a complex Gaussian, so its
modulus also fluctuates around $\sqrt{NS_\theta(f_k)}$. The data consist of the spectrum
(state-dependent) and a random complex Fourier coefficient whose distribution is set by the
spectrum but whose realization is not.

\begin{assumption}[approximately flat bands; in-band observation]\label{as:band}
Let $\mathcal{F}=[0.5,45]$\,Hz contain disjoint bands $B_1,\dots,B_5$, possibly separated by
transition regions. (i) For $f_k\in B_b$, $S_\theta(f_k)\approx s_b(\theta)$. (ii) The
observation consists only of the in-band coefficients $\{X_k:f_k\in\bigcup_bB_b\}$
(equivalently, the signal is first band-passed to $\bigcup_bB_b$). Let
$m_b=\#\{k:f_k\in B_b\}$.
\end{assumption}
\noindent Part (ii) avoids any conflict with the continuity in Assumption~\ref{as:gauss}: the
spectrum may vary continuously everywhere; (i) only requires it to be nearly constant within
each band.

\begin{proposition}[band power is approximately sufficient and the MLE]\label{prop:mle}
Under Assumptions~\ref{as:gauss}--\ref{as:band} (with equality in (i)) and \eqref{eq:whittle},
$P_b=m_b^{-1}\sum_{f_k\in B_b}I_k$ is sufficient for $\bm{s}=(s_b)_b$, $\hat s_b=P_b$ is the
maximum-likelihood estimate, and $\log P_b$ is the MLE of $\log s_b$.
\end{proposition}
\begin{proof}
By (ii) the likelihood involves only in-band bins; by (i),
$-\log p(\bm{x}\mid\bm{s})\approx\sum_b[m_b\log s_b+m_bP_b/s_b]+\text{const}$. The data enter
only through $P_b$; setting the derivative in $s_b$ to zero gives $\hat s_b=P_b$.
\end{proof}
\begin{remark}
Since real spectra are neither exactly flat within bands nor discontinuous between them,
Prop.~\ref{prop:mle} is an approximation under an idealized piecewise-constant model. For an
ideal circular band-pass $\tilde h_b$, Parseval gives
$\frac1N\sum_t((\tilde h_b*\bm{x})_t)^2=\frac{2m_b}{N}P_b$, a per-band constant factor. The
implemented target uses overlapping 65-tap filters (Table~\ref{tab:filters}) and is therefore
not exactly this statistic; the per-patch FFT variant is.
\end{remark}

\begin{proposition}[log power has state-independent noise]\label{prop:var}
Under Lemma~\ref{lem:dft} and Assumption~\ref{as:band}, $P_b=s_bG_b$ with
$G_b\sim\mathrm{Gamma}(m_b,1/m_b)$, so $\Var[P_b]=s_b^2/m_b$,
$\E[\log P_b]=\log s_b+\psi(m_b)-\log m_b$ and $\Var[\log P_b]=\psi'(m_b)\approx1/m_b$.
\end{proposition}
\begin{proof}
A sum of $m$ i.i.d.\ $\mathrm{Exp}(1)$ is $\mathrm{Gamma}(m,1)$, whose logarithm has mean
$\psi(m)$ and variance $\psi'(m)$.
\end{proof}

\begin{proposition}[the implemented $\log(1+P)$]\label{prop:var1p}
Under the same conditions, by the delta method at $G_b=1$, $Y_b=\log(1+P_b)$ satisfies
$\E[Y_b]\approx\log(1+s_b)$ and $\Var[Y_b]\approx m_b^{-1}\bigl(s_b/(1+s_b)\bigr)^2$: close
to $1/m_b$ when $s_b\gg1$ and to $s_b^2/m_b$ when $s_b\ll1$.
\end{proposition}

\begin{assumption}[state persistence]\label{as:persist}
The context $\bm{c}$ and the masked patch share the state $\theta$, and given $\theta$ the
target of the masked patch is independent of $\bm{c}$.
\end{assumption}
\noindent This holds only approximately: neighbouring samples are correlated, and the main
target's 65-tap filters use 32 samples of each neighbouring (visible) patch. The leak-free and
per-patch FFT variants remove the latter dependence and perform equally well.

\begin{proposition}[risk decomposition for any target]\label{prop:decomp}
For any target $Y$ of the masked patch, under Assumption~\ref{as:persist}, the Bayes-optimal
prediction $\E[Y\mid\bm{c}]$ has risk
$\E[\Var(\mu(\theta)\mid\bm{c})]+\E[\sigma^2(\theta)]$ with $\mu(\theta)=\E[Y\mid\theta]$ and
$\sigma^2(\theta)=\Var(Y\mid\theta)$.
\end{proposition}
\begin{proof}
By conditional independence, $\E[Y\mid\bm{c}]=\E[\mu(\theta)\mid\bm{c}]$ and
$\Var(Y\mid\bm{c})=\Var(\mu(\theta)\mid\bm{c})+\E[\sigma^2(\theta)\mid\bm{c}]$; take
expectations over $\bm{c}$.
\end{proof}
\noindent For $Y=\log P_b$, $\sigma^2=\psi'(m_b)$ is constant and the reducible risk is exactly
the posterior uncertainty about $\log s_b$. For the implemented $\log(1+P_b)$ the reducible
part is still uncertainty about the state, but $\sigma^2(\theta)$ varies with the state. For
the waveform $Y=\bm{x}_M$, if $\theta$ were known the optimal predictor is the Wiener
interpolator $\Sigma_{Mc}\Sigma_{cc}^{-1}\bm{c}$ with risk
$|M|^{-1}\operatorname{tr}(\Sigma_{MM}-\Sigma_{Mc}\Sigma_{cc}^{-1}\Sigma_{cM})\le\int S_\theta$:
the irreducible part is the patch's own innovation, it varies with the state, and by
Parseval it weights each frequency by its power. These lead to the predictions in
Sec.~\ref{sec:theory}; they are not theorems about the implementation.

\begin{table}[h]\centering\small
\caption{Measured response of the five fixed target filters (65 taps at 256\,Hz). Last
column: share of each filter's energy within 0--45\,Hz falling into its nominal band.}
\label{tab:filters}
\setlength{\tabcolsep}{3pt}
\begin{tabular}{l c c c c}
\toprule
Dim. & Nominal & Peak & $-3$\,dB band & In nominal band\\
\midrule
``$\delta$'' & 2\,Hz & 4.6\,Hz & 2.9--6.3\,Hz & 31\% (63\% in $\theta$)\\
``$\theta$'' & 6 & 6.0 & 3.7--8.5 & 68\%\\
``$\alpha$'' & 10 & 10.0 & 7.2--12.8 & 71\%\\
``$\beta$'' & 21 & 21.0 & 18.1--23.9 & 100\% (mid-band only)\\
``$\gamma$'' & 38 & 38.0 & 35.1--40.9 & 100\% (mid-band only)\\
\bottomrule
\end{tabular}
\end{table}

\begin{table}[h]\centering\small
\caption{Ideal rectangular bands at our patch length ($N{=}200$, $f_s{=}256$\,Hz,
$\Delta f{=}1.28$\,Hz): number of Fourier bins and the noise of $\log P_b$
(Prop.~\ref{prop:var}). Applies to the standard-band FFT variant, not to the filterbank
target.}
\label{tab:bins}
\setlength{\tabcolsep}{4pt}
\begin{tabular}{l c c c}
\toprule
Band & Range (Hz) & $m_b$ & $\mathrm{sd}[\log P_b]=\sqrt{\psi'(m_b)}$\\
\midrule
$\delta$ & 0.5--4 & 3 & 0.63\\
$\theta$ & 4--8 & 3 & 0.63\\
$\alpha$ & 8--13 & 4 & 0.53\\
$\beta$ & 13--30 & 13 & 0.28\\
$\gamma$ & 30--45 & 12 & 0.29\\
\bottomrule
\end{tabular}
\end{table}

\section{Downstream Data Details}\label{app:data}
\noindent\textbf{Beyond-sleep tasks.} For MA, Mumtaz and PhysioNet-MI we reproduce CBraMod's
\texttt{preprocessing\_*.py} scripts line by line (file ordering, split slices, channel
selection, filtering, resampling to 200\,Hz and window length) and store the samples once;
every model is probed on these samples. FACED's official processed release was not
reachable from our region; we rebuilt it from the public raw BIDS recordings following the
FACED paper (last 30\,s of each clip, 0.05--47\,Hz, common average reference, 250\,Hz,
bad-channel interpolation; no ICA) and then applied CBraMod's FACED preprocessing. Our input
convention adds, for our model only, 0.5--45\,Hz band-pass, resampling to 256\,Hz and
per-recording robust scaling of the same samples, restricted to the 19 standard leads, which
all four tasks contain.

\noindent\textbf{SEED-VIG.} CBraMod's released preprocessing script sorts the 23 session files
lexicographically ('10\_\ldots' before '1\_\ldots'), whereas its paper states the split
1--15/16--19/20--23 by subject. The $R^2$ and RMSE reported in the CBraMod paper imply a
test-label standard deviation of $0.318$ (for both its CBraMod and BIOT rows); the
subject-ordered split gives $0.318$ and the lexicographic split $0.175$, so we use the
subject-ordered split. SEED-VIG has 17 temporal and posterior leads; for our encoder, whose
channel embeddings correspond to 19 fixed 10--20 positions, each SEED-VIG lead receives an
inverse-distance-weighted mix of the embeddings of its two nearest standard positions
(exact positional matches take that slot). Ridge regularization for the probe is selected on
validation $R^2$ from $10^{-2}$ to $10^{8}$.

\section{Earlier Experiments}\label{app:earlier}
The following results come from an earlier stage of this project (original filterbank model
trained with a different launcher, 10-s single-window fine-tuning protocol) and are reported
for completeness; they are not directly comparable with the tables above.
\begin{itemize}
\item \emph{Learnable target collapses.} When the target filters are also learnable, the loss
reaches $4{\times}10^{-4}$ by shrinking the target filters, and downstream features become
uninformative; a fixed target removes this failure mode by construction.
\item \emph{The predictive term carries the signal.} With SIGReg alone (no band-power term),
10-s fine-tuned balanced accuracy is $0.2124$/$0.2737$ on ISRUC/HMC (chance level), versus
$0.7934$/$0.7484$ with the band-power term.
\item \emph{Data scale.} Pretraining on the full ${\sim}17{,}666$\,h of TUEG gave
$0.7899$/$0.7526$ (10-s fine-tuning), statistically indistinguishable from the 2{,}388-hour
model over three independently drawn subsets ($0.7928{\pm}0.0027$/$0.7505{\pm}0.0019$).
\end{itemize}

\clearpage
\setlength{\bibsep}{1pt plus 0pt}
\bibliography{arxiv}

@inproceedings{jiang2024labram,
  title={Large Brain Model for Learning Generic Representations with Tremendous {EEG} Data in {BCI}},
  author={Jiang, Weibang and Zhao, Liming and Lu, Bao-Liang},
  booktitle={International Conference on Learning Representations (ICLR)},
  year={2024}
}

@inproceedings{wang2025cbramod,
  title={{CBraMod}: A Criss-Cross Brain Foundation Model for {EEG} Decoding},
  author={Wang, Jiquan and Zhao, Sha and Luo, Zhiling and Zhou, Yangxuan and Jiang, Haiteng and Li, Shijian and Li, Tao and Pan, Gang},
  booktitle={International Conference on Learning Representations (ICLR)},
  year={2025}
}

@inproceedings{csbrain,
  title={{CSBrain}: A Cross-Scale Spatiotemporal Brain Foundation Model for {EEG} Decoding},
  author={Zhou, Yuchen and Wu, Jiamin and Ren, Zichen and Yao, Zhouheng and Lu, Weiheng and Peng, Kunyu and Zheng, Qihao and Song, Chunfeng and Ouyang, Wanli and Gou, Chao},
  booktitle={Advances in Neural Information Processing Systems (NeurIPS)},
  year={2025},
  note={arXiv:2506.23075}
}

@inproceedings{reve2025,
  title={{REVE}: A Foundation Model for {EEG}---Adapting to Any Setup with Large-Scale Pretraining on 25,000 Subjects},
  author={El Ouahidi, Yassine and Lys, Jonathan and Th{\"o}lke, Philipp and Farrugia, Nicolas and Pasdeloup, Bastien and Gripon, Vincent and Jerbi, Karim and Lioi, Giulia},
  booktitle={Advances in Neural Information Processing Systems (NeurIPS)},
  year={2025},
  note={arXiv:2510.21585}
}

@inproceedings{yang2023biot,
  title={{BIOT}: Biosignal Transformer for Cross-Data Learning in the Wild},
  author={Yang, Chaoqi and Westover, M. Brandon and Sun, Jimeng},
  booktitle={Advances in Neural Information Processing Systems (NeurIPS)},
  year={2023}
}

@article{kostas2021bendr,
  title={{BENDR}: Using Transformers and a Contrastive Self-Supervised Learning Task to Learn from Massive Amounts of {EEG} Data},
  author={Kostas, Demetres and Aroca-Ouellette, St{\'e}phane and Rudzicz, Frank},
  journal={Frontiers in Human Neuroscience},
  volume={15},
  pages={653659},
  year={2021}
}

@inproceedings{cui2024neurogpt,
  title={{Neuro-GPT}: Towards a Foundation Model for {EEG}},
  author={Cui, Wenhui and Jeong, Woojae and Th{\"o}lke, Philipp and Medani, Takfarinas and Jerbi, Karim and Joshi, Anand A. and Leahy, Richard M.},
  booktitle={IEEE International Symposium on Biomedical Imaging (ISBI)},
  year={2024},
  note={arXiv:2311.03764}
}

@inproceedings{wang2024eegpt,
  title={{EEGPT}: Pretrained Transformer for Universal and Reliable Representation of {EEG} Signals},
  author={Wang, Guangyu and Liu, Wenchao and He, Yuhong and Xu, Cong and Ma, Lin and Li, Haifeng},
  booktitle={Advances in Neural Information Processing Systems (NeurIPS)},
  year={2024}
}

@inproceedings{jiang2025neurolm,
  title={{NeuroLM}: A Universal Multi-Task Foundation Model for Bridging the Gap between Language and {EEG} Signals},
  author={Jiang, Wei-Bang and Wang, Yansen and Lu, Bao-Liang and Li, Dongsheng},
  booktitle={International Conference on Learning Representations (ICLR)},
  year={2025},
  note={arXiv:2409.00101}
}

@inproceedings{wang2023brainbert,
  title={{BrainBERT}: Self-Supervised Representation Learning for Intracranial Recordings},
  author={Wang, Christopher and Subramaniam, Vighnesh and Yaari, Adam Uri and Kreiman, Gabriel and Katz, Boris and Cases, Ignacio and Barbu, Andrei},
  booktitle={International Conference on Learning Representations (ICLR)},
  year={2023}
}

@inproceedings{zhang2023brant,
  title={{Brant}: Foundation Model for Intracranial Neural Signal},
  author={Zhang, Daoze and Yuan, Zhizhang and Yang, Yang and Chen, Junru and Wang, Jingjing and Li, Yafeng},
  booktitle={Advances in Neural Information Processing Systems (NeurIPS)},
  year={2023}
}

@inproceedings{foumani2024eeg2rep,
  title={{EEG2Rep}: Enhancing Self-Supervised {EEG} Representation Through Informative Masked Inputs},
  author={Foumani, Navid Mohammadi and Mackellar, Geoffrey and Ghane, Soheila and Irtza, Saad and Nguyen, Nam and Salehi, Mahsa},
  booktitle={Proceedings of the 30th ACM SIGKDD Conference on Knowledge Discovery and Data Mining (KDD)},
  year={2024}
}

@inproceedings{baevski2022data2vec,
  title={data2vec: A General Framework for Self-Supervised Learning in Speech, Vision and Language},
  author={Baevski, Alexei and Hsu, Wei-Ning and Xu, Qiantong and Babu, Arun and Gu, Jiatao and Auli, Michael},
  booktitle={International Conference on Machine Learning (ICML)},
  year={2022}
}

@article{lawhern2018eegnet,
  title={{EEGNet}: A Compact Convolutional Neural Network for {EEG}-based Brain--Computer Interfaces},
  author={Lawhern, Vernon J and Solon, Amelia J and Waytowich, Nicholas R and Gordon, Stephen M and Hung, Chou P and Lance, Brent J},
  journal={Journal of Neural Engineering},
  volume={15},
  number={5},
  pages={056013},
  year={2018}
}

@article{song2023conformer,
  title={{EEG} Conformer: Convolutional Transformer for {EEG} Decoding and Visualization},
  author={Song, Yonghao and Zheng, Qingqing and Liu, Bingchuan and Gao, Xiaorong},
  journal={IEEE Transactions on Neural Systems and Rehabilitation Engineering},
  volume={31},
  pages={710--719},
  year={2023}
}

@article{jing2023sparcnet,
  title={Development of Expert-Level Classification of Seizures and Rhythmic and Periodic Patterns During {EEG} Interpretation},
  author={Jing, Jin and others},
  journal={Neurology},
  volume={100},
  number={17},
  pages={e1750--e1762},
  year={2023}
}

@article{yang2023contrawr,
  title={Self-Supervised {EEG} Representation Learning for Automatic Sleep Staging},
  author={Yang, Chaoqi and Xiao, Cao and Westover, M. Brandon and Sun, Jimeng},
  journal={JMIR AI},
  volume={2},
  pages={e46769},
  year={2023}
}

@inproceedings{peh2022cnntransformer,
  title={Transformer Convolutional Neural Networks for Automated Artifact Detection in Scalp {EEG}},
  author={Peh, Wei Yan and Yao, Yuanyuan and Dauwels, Justin},
  booktitle={44th Annual International Conference of the IEEE Engineering in Medicine \& Biology Society (EMBC)},
  year={2022}
}

@article{li2022ffcl,
  title={Motor Imagery {EEG} Classification Algorithm Based on {CNN-LSTM} Feature Fusion Network},
  author={Li, Hongli and Ding, Man and Zhang, Ronghua and Xiu, Chunbo},
  journal={Biomedical Signal Processing and Control},
  volume={72},
  pages={103342},
  year={2022}
}

@article{song2021sttransformer,
  title={Transformer-Based Spatial-Temporal Feature Learning for {EEG} Decoding},
  author={Song, Yonghao and Jia, Xueyu and Yang, Lie and Xie, Longhan},
  journal={arXiv preprint arXiv:2106.11170},
  year={2021}
}

@inproceedings{he2022mae,
  title={Masked Autoencoders Are Scalable Vision Learners},
  author={He, Kaiming and Chen, Xinlei and Xie, Saining and Li, Yanghao and Doll{\'a}r, Piotr and Girshick, Ross},
  booktitle={IEEE/CVF Conference on Computer Vision and Pattern Recognition (CVPR)},
  year={2022}
}

@inproceedings{assran2023ijepa,
  title={Self-Supervised Learning from Images with a Joint-Embedding Predictive Architecture},
  author={Assran, Mahmoud and Duval, Quentin and Misra, Ishan and Bojanowski, Piotr and Vincent, Pascal and Rabbat, Michael and LeCun, Yann and Ballas, Nicolas},
  booktitle={IEEE/CVF Conference on Computer Vision and Pattern Recognition (CVPR)},
  year={2023}
}

@inproceedings{bao2022beit,
  title={{BEiT}: {BERT} Pre-Training of Image Transformers},
  author={Bao, Hangbo and Dong, Li and Piao, Songhao and Wei, Furu},
  booktitle={International Conference on Learning Representations (ICLR)},
  year={2022}
}

@article{balestriero2025lejepa,
  title={{LeJEPA}: Provable and Scalable Self-Supervised Learning Without the Heuristics},
  author={Balestriero, Randall and LeCun, Yann},
  journal={arXiv preprint arXiv:2511.08544},
  year={2025}
}

@article{obeid2016tuh,
  title={The Temple University Hospital {EEG} Data Corpus},
  author={Obeid, Iyad and Picone, Joseph},
  journal={Frontiers in Neuroscience},
  volume={10},
  pages={196},
  year={2016}
}

@article{khalighi2016isruc,
  title={{ISRUC-Sleep}: A Comprehensive Public Dataset for Sleep Researchers},
  author={Khalighi, Sirvan and Sousa, Teresa and Santos, Jos{\'e} Moutinho and Nunes, Urbano},
  journal={Computer Methods and Programs in Biomedicine},
  volume={124},
  pages={180--192},
  year={2016}
}

@article{alvarez2021hmc,
  title={Inter-Database Validation of a Deep Learning Approach for Automatic Sleep Scoring},
  author={Alvarez-Estevez, Diego and Rijsman, Roselyne M.},
  journal={PLOS ONE},
  volume={16},
  number={8},
  pages={e0256111},
  year={2021}
}

@article{zyma2019mentalarith,
  title={Electroencephalograms during Mental Arithmetic Task Performance},
  author={Zyma, Igor and Tukaev, Sergii and Seleznov, Ivan and Kiyono, Ken and Popov, Anton and Chernykh, Mariia and Shpenkov, Oleksii},
  journal={Data},
  volume={4},
  number={1},
  pages={14},
  year={2019}
}

@inproceedings{wei2022maskfeat,
  title={Masked Feature Prediction for Self-Supervised Visual Pre-Training},
  author={Wei, Chen and Fan, Haoqi and Xie, Saining and Wu, Chao-Yuan and Yuille, Alan and Feichtenhofer, Christoph},
  booktitle={Proceedings of the IEEE/CVF Conference on Computer Vision and Pattern Recognition (CVPR)},
  year={2022}
}

@article{hsu2021hubert,
  title={{HuBERT}: Self-Supervised Speech Representation Learning by Masked Prediction of Hidden Units},
  author={Hsu, Wei-Ning and Bolte, Benjamin and Tsai, Yao-Hung Hubert and Lakhotia, Kushal and Salakhutdinov, Ruslan and Mohamed, Abdelrahman},
  journal={IEEE/ACM Transactions on Audio, Speech, and Language Processing},
  volume={29},
  pages={3451--3460},
  year={2021}
}

@article{yu2026fame,
  title={Understanding and Correcting Low-Frequency Bias in {EEG} Foundation Model},
  author={Yu, Junjie and Deng, Zihan and Zhang, Jianyu and Mu, Junrong and An, Jiahui and Ma, Wenxiao and Lu, Ziling and Wang, Yue and Zhu, Yan and Lou, Kexin and Liu, Quanying},
  journal={arXiv preprint arXiv:2608.01898},
  year={2026}
}

@article{kulkarni2026manas2,
  title={{MANAS-2}: Constrained Reconstruction for {EEG} Foundation Models},
  author={Kulkarni, Arvasu and Mishra, Aditya Ray and Lahiri, Jeet Bandhu and Jain, Mahir and Runwal, Parshva and Saini, Lakshya and Panwar, Siddharth and Singh, Sandeep},
  journal={arXiv preprint arXiv:2609.13717},
  year={2026}
}

@article{sukhbaatar2026bandvq,
  title={{BandVQ}: Band-Wise Vector-Quantized {EEG} Foundation Model},
  author={Sukhbaatar, Jamiyan and Imamura, Satoshi and Tanaka, Toshihisa},
  journal={arXiv preprint arXiv:2605.24921},
  year={2026}
}

@article{bettinardi2025bioserenity,
  title={{BioSerenity-E1}: a self-supervised {EEG} model for medical applications},
  author={Bettinardi, Ruggero G. and Rahmouni, Mohamed and Gimenez, Ulysse},
  journal={arXiv preprint arXiv:2503.10362},
  year={2025}
}

@book{brillinger1981,
  title={Time Series: Data Analysis and Theory},
  author={Brillinger, David R.},
  publisher={Holden-Day},
  year={1981}
}

@article{whittle1953,
  title={Estimation and Information in Stationary Time Series},
  author={Whittle, Peter},
  journal={Arkiv f{\"o}r Matematik},
  volume={2},
  number={5},
  pages={423--434},
  year={1953}
}

@article{chen2023faced,
  title={A Large Finer-Grained Affective Computing {EEG} Dataset},
  author={Chen, Jingjing and Wang, Xiaobin and Huang, Chen and Hu, Xin and Shen, Xinke and Zhang, Dan},
  journal={Scientific Data},
  volume={10},
  pages={740},
  year={2023}
}

@article{zheng2017seedvig,
  title={A Multimodal Approach to Estimating Vigilance Using {EEG} and Forehead {EOG}},
  author={Zheng, Wei-Long and Lu, Bao-Liang},
  journal={Journal of Neural Engineering},
  volume={14},
  number={2},
  pages={026017},
  year={2017}
}

@article{schalk2004bci2000,
  title={{BCI2000}: A General-Purpose Brain-Computer Interface ({BCI}) System},
  author={Schalk, Gerwin and McFarland, Dennis J. and Hinterberger, Thilo and Birbaumer, Niels and Wolpaw, Jonathan R.},
  journal={IEEE Transactions on Biomedical Engineering},
  volume={51},
  number={6},
  pages={1034--1043},
  year={2004}
}

@misc{mumtaz2016mdd,
  title={{MDD} Patients and Healthy Controls {EEG} Data (New)},
  author={Mumtaz, Wajid},
  howpublished={figshare, Dataset},
  year={2016}
}

@book{berry2017aasm,
  title={The {AASM} Manual for the Scoring of Sleep and Associated Events: Rules, Terminology and Technical Specifications},
  author={Berry, Richard B. and others},
  publisher={American Academy of Sleep Medicine},
  year={2017}
}

\end{document}